\documentclass[letterpaper, 10 pt]{article}

\usepackage{amsmath,amsthm,amssymb} % basic

\usepackage[dvipsnames,table]{xcolor} % dvipsnames option for more colors - need to load this package early

\usepackage{cite} % to collapse [1],[2],[3],[4] --> [1-4]
\usepackage{xspace} % for putting space after new commands
\usepackage{mathtools}
\usepackage{cuted} % for equations spanning two columns
\usepackage{lipsum} % dummy text
\usepackage{widetext}

\usepackage{graphicx} % support the \includegraphics command and options
\usepackage{subcaption} % for subfigures
\usepackage{multirow} % to have more control on subfigure positioning
\usepackage{float} % to be able to put [H] in figures
\usepackage{tikz}

\usepackage{enumitem}

\usepackage{algorithm,algorithmic}

\usepackage{booktabs} % for much better looking tables
\usepackage{tabu} % to have further control on format of tables 
\usepackage{makecell} % to be able to break line inside table cell
\usepackage{footnote}
\makesavenoteenv{tabular} % to be able to put footnotes in taular environment

\usepackage[hyphens]{url}
\usepackage{etoolbox} % to suppress "References" in the environment thebibliography
\usepackage[margin=1in]{geometry} % change default article style

\newcommand{\Ebb}{\mathbb{E}} % expectation
\newcommand{\Rbb}{\mathbb{R}} % real numbers
\newcommand{\Dcal}{\mathcal{D}}
\newcommand{\Hcal}{\mathcal{H}}

\newcommand{\Scal}{\mathcal{S}} % generic set
\newcommand{\Ncal}{\mathcal{N}} % set of neighbor nodes
\newcommand{\Wcal}{\mathcal{W}}

\newcommand{\Ucal}{\mathcal{U}} % uniform distribution

\DeclareMathOperator*{\argmin}{argmin} % thin space, limits underneath in displays

\newcommand{\ones}{\mathbf{1}} % ones vector
\newcommand{\avg}{\text{avg}}
\newcommand\norm[1]{{\left\lVert#1\right\rVert}} % for norms ||.||
\newcommand\normts[1]{{\left\lVert#1\right\rVert_2^2}} % norm 2 squared
\newcommand\abs[1]{{\left|#1\right|}} % for abs value |.|
\newcommand\paren[1]{{\left(#1\right)}} % for parentheses (.)
\newcommand\Sbraces[1]{{\left[#1\right]}} % for square braces [.]
\newcommand\Abraces[1]{{\left\langle#1\right\rangle}} % for angle braces <.>
\newcommand{\algindent}{\hspace{\algorithmicindent}} % one algorithmic indent

\newcommand{\ram}{\rightarrow} % \rightarrow directly in mathmode
\newcommand\numberthis{\addtocounter{equation}{1}\tag{\theequation}} % to number only some eqs in \align{}

\renewcommand{\u}{{\mathbf{u}}} 
\renewcommand{\v}{{\mathbf{v}}} 
\newcommand{\x}{{\mathbf{x}}} 
\newcommand{\y}{{\mathbf{y}}} 
\newcommand{\z}{{\mathbf{z}}} 
\newcommand{\ox}{{\overline{\x}}} 
\newcommand{\oz}{{\overline{\z}}} 
 
\newcommand{\zs}{{\mathbf{z}^*}} 
\newcommand{\os}{{\overline{\sigma}}} 
 
\newcommand{\sumin}{{\sum_{i=1}^n}} 
\newcommand{\noo}{{\frac{1}{n}\ones\ones^T}} 
 
\newcommand{\EbbSt}{{\Ebb_{\Scal_t}}} 

\newtheorem{theorem}{Theorem}

\newtheorem{lemma}[theorem]{Lemma}
\theoremstyle{definition}

\newtheorem{fact}{Fact}

\newtheorem{assumption}{Assumption}

\title{\LARGE \bf FedDec: Peer-to-peer Aided Federated Learning}

\author{Marina Costantini$^{1,*}$ \and Giovanni Neglia$^2$ \and Thrasyvoulos Spyropoulos$^{1,3}$}
\date{\normalsize{
    $^1$EURECOM, Sophia Antipolis \qquad
    $^2$Inria \& Université Côte d'Azur\\
    $^3$Technical University of Crete \qquad
    $^*$Correspondence: marina.costantini@eurecom.fr
    } }

\begin{document}

\maketitle

% +++++++++++++++++++++++++++++++++++
% UNCOMMENT THE TWO LINES BELOW BEFORE SUBMISSION (ORIGINAL)

\thispagestyle{empty}
\pagestyle{empty}
% +++++++++++++++++++++++++++++++++++
% COMMENT THE TWO LINES BELOW BEFORE SUBMISSION (ADDED BY MC FOR PAGE NUMBERING)

% \thispagestyle{plain}
% \pagestyle{plain}
% +++++++++++++++++++++++++++++++++++

\begin{abstract}

Federated learning (FL) has enabled training machine learning models exploiting the data of multiple agents without compromising privacy. 
However, FL is known to be vulnerable to data heterogeneity, partial device participation, and infrequent communication with the server, which are nonetheless three distinctive characteristics of this framework.
While much of the recent literature has tackled these weaknesses using different tools, only a few works have explored the possibility of exploiting inter-agent communication to improve FL's performance.
In this work, we present FedDec, an algorithm that interleaves peer-to-peer communication and parameter averaging (similar to decentralized learning in networks) between the local gradient updates of FL.
We analyze the convergence of FedDec under the assumptions of non-iid data distribution, partial device participation, and smooth and strongly convex costs, and show that inter-agent communication alleviates the negative impact of infrequent communication rounds with the server by reducing the dependence on the number of local updates $H$ from $O(H^2)$ to $O(H)$. 
Furthermore, our analysis reveals that the term improved in the bound is multiplied by a constant that depends on the spectrum of the inter-agent communication graph, and that vanishes quickly the more connected the network is.
We confirm the predictions of our theory in numerical simulations, where we show that FedDec converges faster than FedAvg, and that the gains are greater as either $H$ or the connectivity of the network increase.

\end{abstract}

\section{Introduction}

Federated learning (FL) is a recent machine learning framework that allows multiple agents, each of them with their own dataset, to train a model collaboratively without sharing their data \cite{mcmahan2017communication, karimireddy2020scaffold, li2020federated, qu2021federated}. 
The \textit{federated} setting assumes that all agents are connected to a server that can communicate with each of them and that is in charge of aggregating the agents' updates to obtain the global model. 
This is similar to \textit{parallel distributed} (PD) model training \cite{xiao2019dscovr, recht2011hogwild, liu2014asynchronous, smith2018cocoa}, with one crucial difference: in the latter, the agents send \textit{gradients} to the central server to update the parameter value with a \textit{gradient step}, while in FL the agents send their \textit{own local parameters} for the server to \textit{average} them.
This has an impact on the communication frequency required by each framework: in PD one round of communication between (usually all) the agents and the server has to happen every time a (mini-batch) stochastic gradient descent (SGD) step is taken at the nodes, while in FL (i) multiple SGD updates can happen before a new server communication round takes place (which in FL literature are usually called \textit{local updates}), and (ii) not all devices need to engage in the server communication round (which is known as \textit{partial participation}). 
This makes FL a much more suitable option for settings with a large number of agents and a limited communication bandwidth with the server.

In contrast to the approaches described above, the \textit{decentralized} setting does not rely on a central server for the aggregation of the nodes' updates.
Instead, it assumes that the agents are interconnected in a network and each of them can exchange optimization values (either parameters or gradients, depending on the algorithm) with its direct neighbors \cite{nedic2009distributed, shi2015extra, duchi2011dual, scaman2017optimal, lian2017can, koloskova2019decentralized, uribe2020dual}. 
In the decentralized setting, every node performs an averaging step of all its neighbors' received values before taking a new gradient step.  
Algorithms for this setting are designed such that the local parameters of all nodes converge to the global minimizer, while in FL it is the central server who keeps track of the most recent parameter value and broadcasts it to all agents every once in a while. 

%%%%%%%%%%%%%%%%%%%%%%%%%%%%%%%%%%%%%%%%%%%%%%%%%%%%%%%%%%%%%%%%%%%%%%%%%%%%%%%%%%%%%%%%%%%%%%%%%%%%%%%%%%%

The attractive feature of FL of allowing to have server communication rounds every once in a while 
% (in contrast to PD) 
comes at a cost: the more infrequent the server communication rounds are (i.e., the more local updates are performed at the agents), the slower is the convergence \cite{li2019convergence, zhang2016parallel}.
For this reason, in this paper we propose to exploit inter-agent communication to reduce the negative impact of infrequent server communication rounds.
Given that (i) each agent is expected to have much fewer neighbors than the total number of agents, and (ii) short-range inter-agent communications allow for spectrum reuse, agents can communicate much more often between them than with the server \cite{hellaoui2018aerial, asadi2014survey}.

We propose FedDec, an FL algorithm where the agents can exchange and average their current parameters with those of their neighbors in between the local SGD steps. 
We show that this modification reduces the dependence of the convergence bound on the number of local SGD steps $H$ from $O(H^2)$ \cite{li2019convergence, stich2018local} to $O(H)$ (Theorem \ref{theo:convergence_non-iid_partial_participation}).
Furthermore, we show that, in our analysis, the extra $H$ factor is replaced  by a value $\alpha$ that depends on the spectrum of the graph defining the inter-agent communication. 
Since the value of $\alpha$ quickly decreases as the network becomes more connected, our result indicates that for mildly connected networks $H$ can be increased without severely hurting convergence speed (or conversely, for fixed $H$, FedDec will be faster than FedAvg \cite{mcmahan2017communication}, its counterpart without inter-agent communication).

% --------------------------------------------------------------------
% While a few works \cite{hosseinalipour2022multi, koloskova2019decentralized, yemini2022semi, lin2021semi, chou2021efficient, } have considered peer-to-peer communication within FL, none of them has characterized analytically how the inter-agent communication reduces the impact of local updates on convergence.
% --------------------------------------------------------------------
% While a few works in the past have demonstrated in simulations its advantages of peer-to-peer communication within FL settings has been considered by in numerical results \cite{yemini2022semi, lin2021semi, chou2021efficient, hosseinalipour2022multi}, none of these works has characterized analytically how the inter-agent communication reduces the impact of local updates on convergence, and instead they have demonstrated its advantages only in numerical results.
% --------------------------------------------------------------------
Peer-to-peer communication within FL has been considered a few times in the past \cite{yemini2022semi, lin2021semi, chou2021efficient, hosseinalipour2022multi, koloskova2020unified}.
These works either analyze very general distributed settings that have FL with inter-agent communication as a particular case \cite{hosseinalipour2022multi, koloskova2020unified}, 
% These works either analyze very general distributed settings that account for inter-agent communications in FL \cite{hosseinalipour2022multi, koloskova2020unified} 
or show that FL with inter-agent communication converges at the same rate as standard FL, and outperforms it in simulations \cite{yemini2022semi, lin2021semi, chou2021efficient}.
However, none of them has characterized analytically how inter-agent communication reduces the impact of local updates on convergence, and in particular, how this reduction depends on the inter-agent connectivity.

%%%%%%%%%%%%%%%%%%%%%%%%%%%%%%%%%%%%%%%%%%%%%%%%%%%%%%%%%%%%%%%%%%%%%%%%%%%%%%%%%%%%%%%

Our contributions can be summarized as follows:
\begin{itemize}
    \item We introduce FedDec, an FL algorithm where the agents can average their parameters with those of their neighbors in between the SGD steps. 
    Our model accounts for failures in the inter-agent communication links, so that only a few (or even none) of the parameters of a node's neighbors may be averaged at some iterations.    
    
    \item We prove that, for non-iid data, partial device participation, and smooth and strongly convex objectives, FedDec converges at the $O(1/T)$ rate (where $T$ is the total number of iterations executed) of FL algorithms that do not account for inter-agent communication \cite{li2019convergence, stich2018local}, but improves the dependence on the number of local updates $H$ from $O(H^2)$ to $O(H)$.

    \item Furthermore, we show that the improved term is multiplied by a quantity that depends on the spectrum of the inter-agent communication network, and which quickly vanishes the more connected the network is.
    % Thus, we not only reduce the order of the dependence on $H$, but also show that this term becomes smaller the more connected the network is.
    % show explicitly how the network topology, characterized by the spectral radius of the expected value of the neighbor mixing matrix, impacts on convergence speed.
    % To the best of our knowledge, this is the first time that the role of this eigenvalue in convergence is shown for an algorithm exploiting peer-to-peer communication within FL with partial device participation and time-changing networks. 
   
    \item We support our theoretical findings with numerical simulations, where we confirm that the performance of FedDec with respect to FedAvg \cite{mcmahan2017communication} increases with both $H$ and the connectivity of the network. 
    % demonstrate in numerical simulations that FedDec outperforms FedAvg \cite{mcmahan2017communication}, arguably the most widespread FL algorithm, and is more robust against infrequent communication rounds with the server. 
    
\end{itemize}

\section{System Model and the FedDec Algorithm} \label{sec:model}

\begin{figure}[t]
\centering
\includegraphics[width=0.6\linewidth]{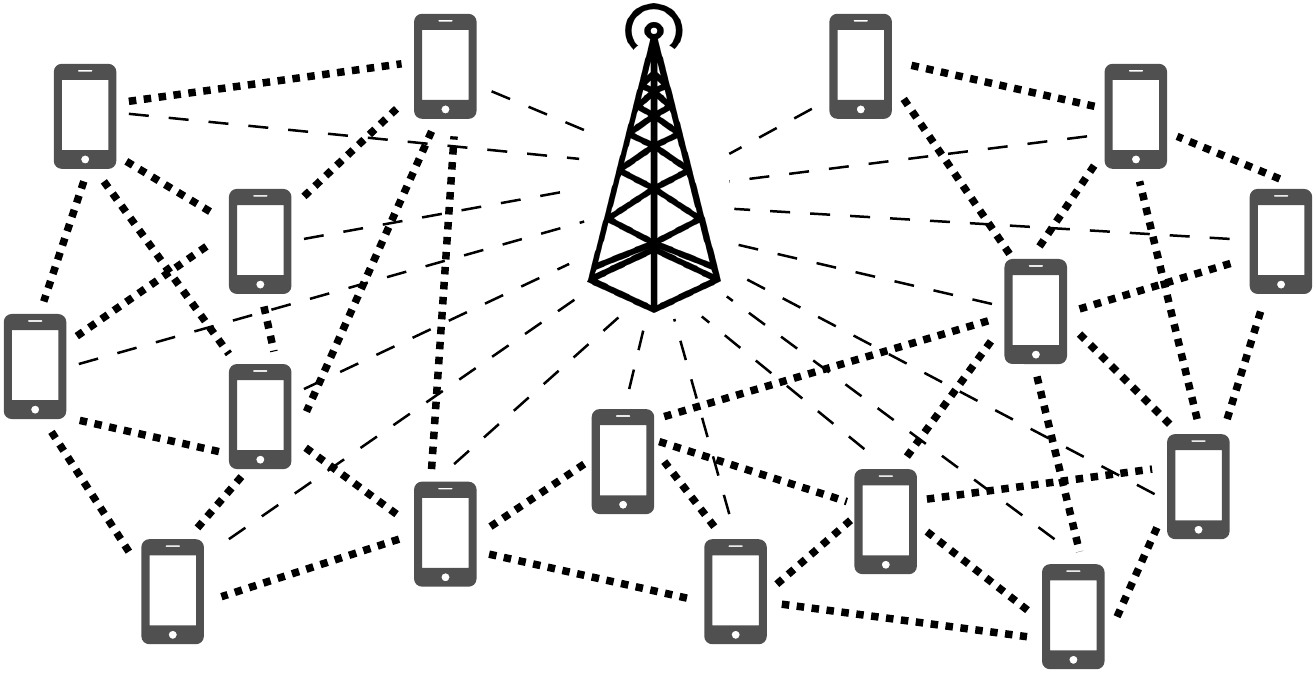}
\caption{The FedDec setting. Peer-to-peer communication links are shown in dotted lines (high-bandwidth links). Communication links between the agents and the server are shown in dashed lines (low-bandwidth links). }
\label{fig:feddec_setting}
\end{figure}

We consider a system where $n$ agents can exchange messages with a central server and also with some other nearby agents. 
We assume that the inter-agent communication links may fail at some iterations (e.g. due to outage), but when all links are active the agents form a connected network (see Figure \ref{fig:feddec_setting}).
Each node $i \in [n]$ has a local cost $F_i: \Rbb^d \ram \Rbb$,
\[ F_i(\z) := \Ebb_{\psi_i \sim \Dcal_i} F_i(\z, \psi_i), \]
where $\Dcal_i$ can be an underlying local data distribution from where new samples (or mini-batches) are drawn each time an SGD step is taken, or the uniform distribution over a static dataset.
Note that the $\Dcal_i$ can be different at each node. 
The objective of the nodes and the server is to find the minimizer
\[ \z^* := \argmin_{\z \in \Rbb^d} f(\z), \qquad f(\z) := \frac{1}{n} \sumin F_i(\z) \]
under the constraints that nodes can only communicate with their direct neighbors (high-bandwidth links) and every once in a while they can get a request from the server to send their current parameter values (low-bandwidth links). 
In wireless settings, these capacities are imposed by the shared nature of the cellular medium. 
While the communication with the server is constrained by the bandwidth available, device-to-device communications in the short range allows for spectrum reuse, and thus for higher throughput \cite{hellaoui2018aerial, asadi2014survey}. 

At each server communication round, the server samples the devices uniformly\footnote{Our analysis is readily extendable to the case where the server samples with non-uniform probabilities $\{p_i\}_{i=1}^n$, in which case the cost becomes $f(\z) {=} \sumin p_i F_i(\z)$ and the term 
% $\os^2/n$ 
$\frac{\os^2}{n}$ 
in Theorem \ref{theo:convergence_non-iid_partial_participation} becomes $\sumin p_i^2 \sigma_i^2$.} 
at random with replacement to form an index pool $\Scal_t$ of devices that it will poll during that round. We assume $|\Scal_t| = K \; \forall \, t$.
It then averages the parameters of all $j \in \Scal_t$ and broadcasts the new value to \textit{all} nodes in the network. 
Due to the limited bandwidth, we assume partial participation, i.e. $K \ll n$. 
We assume that the server aggregation rounds happen every $H$ local updates, and we call $\Hcal = \{t: t \text{ modulo } H=0\}$ the set of those times.

One local update of FedDec for a node $i$ consists on 
(i) taking an SGD step, 
(ii) for all active links (i.e. $\forall j{:} \, W^t_{ij} > 0$), exchanging the new parameter value with its neighbors, and
(iii) combining all new values (including its own) with weights $W^t_{ij}$ to form the new iterate.
We call this algorithm FedDec, and the precise steps are shown in Algorithm \ref{alg:feddec}. 

% -----------------  ALGORITHM  -----------------
\begin{algorithm}[t]
    \caption{Peer-to-peer aided FL (FedDec)}
    \label{alg:feddec}
    \begin{algorithmic}[1]
        \STATE Initialize $\z_i^1 = \z^1 \; \forall i \in [n]$ and let $\eta_t = \frac{2}{\mu(t+\gamma)}$.
        \FOR{$t = 1,\ldots,T$ all agents $i \in [n]$}  
            \STATE Sample mixing matrix $W^t \sim \Wcal$
            \STATE Sample mini-batch $\xi_i^t$  and compute $\nabla F_i(\z_i^t, \xi_i^t)$ 
            \STATE Update local parameter $\x_i^{t+\frac{1}{2}} = \z_i^t - \eta_t \nabla F_i(\z_i^t, \xi_i^t)$
            \STATE Average with neighbors $\x_i^{t+1} = \sum_{j=1}^n W^t_{ij} \x_j^{t+\frac{1}{2}}$
            \STATE \textbf{If} $t+1 \in \Hcal$ \textbf{then}
            \STATE \algindent Server samples $\Scal_t = \{j_\ell: j_\ell \sim \Ucal([n])\}_{\ell=1}^K$
            % \STATE \algindent Server samples $\Scal_t = \{j_\ell: j_\ell \sim \Pcal\}_{\ell=1}^K$ % non-uniform sampling 
            \STATE \algindent it computes $\z^{t+1} = \frac{1}{K} \sum_{\ell = 1}^K \x_{j_\ell}^{t+1}$
            \STATE \algindent and broadcasts so that $\z_i^{t+1} = \z^{t+1} \; \forall i \in [n]$
            \STATE \textbf{otherwise}
            \STATE \algindent $\z_i^{t+1} = \x_i^{t+1}$
        \ENDFOR
        \STATE Output $\z^{T+1}$
    \end{algorithmic}
\end{algorithm}
% -----------------------------------------------

With slight abuse of notation, the stochastic gradient of a node $i$ computed on a mini-batch 
$\xi_i = \{\psi_i^j: \psi_i^j \sim \Dcal_i \}_{j=1}^m$ of size $m$ is given by 
$\nabla F_i(\z_i, \xi_i) = \frac{1}{m} \sum_{j=1}^m \nabla F_i(\z_i, \psi_i^j)$.
For simplicity, we will assume that the number of iterations $T$ satisfies $(T$ modulo $H) = 0$, so that the outputted value in Alg. \ref{alg:feddec} is the current parameter at all nodes.
We will also take the following assumptions, which are standard in the literature \cite{koloskova2019decentralized, stich2018local, li2019convergence, lin2021semi, chou2021efficient, yemini2022semi}.

% A A A A A A A A A A A A A A A A A A A A A A A A A A A A 
\begin{assumption} \label{assu:setting}
    We assume the following $\forall F_i(\z), i \in [n]$: 

    1) \textbf{$L$-smoothness and $\mu$-strong convexity:}
    \begin{align} 
        F_i(\y) &\leq F_i(\x) + \Abraces{\nabla F_i(\x), \y-\x} + (L / 2) \normts{\y-\x} \label{eq:smoothness}, \\ 
        F_i(\y) &\geq F_i(\x) + \Abraces{\nabla F_i(\x), \y-\x} + (\mu / 2) \normts{\y-\x}. \label{eq:strong_conv}
    \end{align}

    2) \textbf{Bounded variance of the local gradients:}
    \[\Ebb \norm{\nabla F_i(\x, \xi_i) - \nabla F_i(\x)}_2^2 \leq \sigma_i^2. \]
    
   3) \textbf{Bounded energy of the local gradients:} 
   \begin{equation} \label{eq:bounded_energy_gradients}
       \Ebb \norm{\nabla F_i(\x, \xi_i)}_2^2 \leq G^2 \text{ for } i \in [n].
   \end{equation}
    
\end{assumption}

We remark that \eqref{eq:strong_conv} implies $\norm{\nabla F_i(\z_i)}_2 \geq \mu \norm{\z_i - \z_i^*}_2$ (see definition of $\z_i^*$ below). Therefore, in order to satisfy \eqref{eq:bounded_energy_gradients} we must additionally assume that the parameter iterates $\z_i$ belong to a bounded set throughout the iterations.

Note that $L$-smoothness implies
\begin{equation} \label{eq:smooth_dFxdFx_leq_Lxy}
    \norm{\nabla F_i(\x) - \nabla F_i(\y)} \leq L \norm{\x-\y}_2,
\end{equation}
\begin{equation} \label{eq:smooth_df2_leq_ffs}
  \normts{\nabla f(\z)} {\,=\,} \norm{\nabla f(\z) {-} \nabla f(\zs)}_2^2 {\,\leq\,}  2L(f(\z){-}f(\zs)).
\end{equation}

Furthermore, the local gradient's bounded variance implies
\begin{equation} \label{eq:bound_avg_sigma_2}
    \Ebb \norm{\frac{1}{n} \sumin \paren{\nabla F_i(\z_i^t) - \nabla F_i(\z_i^t,\xi_i^t)}}_2^2 
    \leq \frac{\os^2}{n},
\end{equation}
with $\os^2 = \frac{1}{n} \sumin \sigma_i^2$.

We quantify the degree of heterogeneity (or non-iidness) between the local functions through the quantity 
\[ \Gamma = \frac{1}{n} \sumin (F_i(\zs) - F_i(\z_i^*)), \text{ where } \; \z_i^* = \argmin_{\z} F_i(\z). \]

FedDec uses two parameters to track the updates, $\z_i$ and $\x_i$ (see Alg. \ref{alg:feddec}). 
%
% FedDec uses two parameters to track the updates: $\z_i$, the parameters just before the SGD and neighbor averaging steps, and $\x_i$, the parameters just after those steps. These values may be equal or not depending on $t$ and whether it is a server communication round (see Alg. \ref{alg:feddec}). 
%
% For both of these parameters,
We define 
\[ \ox^t = \frac{1}{n} \sum_{i=1}^n \x_i^t, 
\qquad \oz^t = \frac{1}{n} \sum_{i=1}^n \z_i^t, \]
which will be useful in the analysis. 
Note that $\oz^t = \ox^t$ only when $t \notin \Hcal$. 
Otherwise, if $t \in \Hcal$, we have that the equality holds only in expectation:
\begin{equation} \label{eq:exp_oz_is_ox}
    \Ebb_{\Scal_t} \oz^t 
    = \Ebb_{\Scal_t} \frac{1}{n} \sumin \z_i^t 
    = \frac{1}{n} \sumin \frac{1}{K} \sum_{\ell=1}^K \Ebb_{\Scal_t} \x_{j_\ell}^t 
    = \ox^t.
\end{equation} 

Lastly, we assume the following about the $W^t = \{W^t_{ij}\}$.

\begin{assumption} \label{assu:mixing_matrix}
    The averaging matrices $W^t \in \Rbb^{n \times n}$ are iid random variables drawn from a distribution $\Wcal$ of matrices that (i) are symmetric, (ii) are doubly stochastic, and (iii) have $W^t_{ij} \geq 0$ if agents $i$ and $j$ are connected and $W^t_{ij} = 0$ otherwise. Note that this implies that $\forall \, W \in \Wcal: W\ones = \ones$, $\ones^T W = \ones^T$.
    Additionally, we require that the eigenvalues of $\Ebb_W \Sbraces{WW^T}$ satisfy $1 = \lambda_1 > \abs{\lambda_2} \geq \ldots \geq \abs{\lambda_n}$.
\end{assumption}

% Using these assumptions and notation, 
In the next section we prove that FedDec converges as $O(1/T)$, similarly to other FL algorithms taking the same assumptions, but it reduces the negative impact of local updates by replacing an $H^2$ factor \cite{li2019convergence, stich2018local}, with $H\alpha$, where $\alpha$ is a quantity that decreases quickly as the inter-agent communication network becomes more connected.
\section{Convergence Analysis}

The following theorem establishes the convergence rate of FedDec and constitutes our main result. 

\begin{theorem}[] \label{theo:convergence_non-iid_partial_participation}
Under Assumptions \ref{assu:setting}, \ref{assu:mixing_matrix}, and for diminishing stepsize $\eta_t = \frac{2}{\mu(\gamma+t)}$, FedDec in Algorithm \ref{alg:feddec} converges as 
    \[ \Ebb[f(\oz^t)] - f(\zs) \leq \frac{L}{\gamma + t} \paren{ \frac{2B}{\mu^2} + \frac{(\gamma+1)}{2} \normts{\z^1 - \zs} } \]
    where 
    \begin{gather*}
        \gamma = \max\{8(L/\mu)-1, H\} \\
        B = \paren{4/K + 8} \alpha H G^2 + 6L\Gamma + \os^2/n \\
        \alpha = |\widehat{\lambda}_2| \, / \, (1 - |\widehat{\lambda}_2|)
    \end{gather*}
    and $|\widehat{\lambda}_2| =  |\lambda_2 \paren{ \Ebb_W \Sbraces{W W^T} }| $.
\end{theorem}

The theorem shows that factors like the energy and the variance of the local gradients, the heterogeneity of the local functions, and the distance of the starting point to the optimum all slow down convergence, which are known facts.
However, this bound also shows how inter-agent communication partially mitigates the negative impact of local updates: the term where $H$ appears decreases with $\alpha$, and therefore decreases very fast with $|\widehat{\lambda}_2|$ (see Figure \ref{fig:lambda_dependence}).

Note that if all inter-agent communication links are assumed to be always active, then $W^t = W$ is a fixed matrix and $|\widehat{\lambda}_2| = \abs{\lambda_2}^2$.
For any given heuristic to construct $W$ (e.g. based on the Laplacian of the graph \cite{xiao2004fast}), the value of $\abs{\lambda_2}$ is, in general, lower the more connected the network is (see Table \ref{tab:values_lambda} in Section \ref{sec:simulations}). 
Therefore, the more densely connected the network is, the faster FedDec is expected to converge. 
In fact, the averaging weights $W_{ij}$ can be designed in order to minimize $\abs{\lambda_2}$ (and thus maximize the speedup from inter-agent communication) using eigenvalue optimization techniques \cite{boyd2006randomized}.

Comparing the bound of Theorem \ref{theo:convergence_non-iid_partial_participation} with that of Theorem 2 in \cite{li2019convergence}, obtained for the same setting but without allowing inter-agent communication, we note that the dependence of the first term in $B$ on the number of local iterations $H$ drops from $O(H^2)$ in \cite{li2019convergence} to $O(H)$ in our theorem. 
This suggests that \textit{the peer-to-peer communication of FedDec reduces the impact of the infrequent communication rounds with the server}, and thus its convergence should be less affected than that of FedAvg as $H$ increases. 
We verify this behavior in our simulations in Section \ref{sec:simulations}.

To prove the theorem, we will need to bound the quantity $\normts{\oz^t - \zs}$. 
For this we will decompose the term and bound $\normts{\oz^t - \ox^t}$ and $\normts{\ox^t - \zs}$ separately.
The following lemmas present these intermediate results, and we prove Theorem \ref{theo:convergence_non-iid_partial_participation} at the end of the section. 
The proofs of the lemmas are given in Appendix \ref{app:proofs_of_lemmas}.

\begin{figure}[t]
\centering \includegraphics[width=0.5\linewidth]{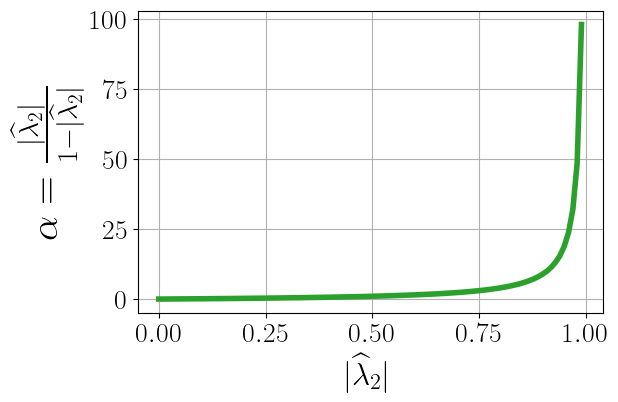}
\caption{Dependence of $\alpha$ with $|\widehat{\lambda}_2|$. For networks moderately connected (see Section \ref{sec:simulations}) we can expect that $\alpha \ll H$, and thus also that FedDec will be faster than FedAvg.}
\label{fig:lambda_dependence}
\end{figure}

% :::::::::::::::::::::::::::::::::::::::::::::::::::::::::::::::
\begin{lemma} \label{lem:bound_avg_to_optim}
    For FedDec with stepsize $\eta_t \leq \frac{1}{4L}$ it holds 
    \[ \Ebb {\normts{\ox^{t+1} - \zs}} \leq  
    (1-\mu \eta_t) \Ebb \norm{\oz^t - \zs}_2^2
    + \frac{2}{n} \Ebb \sumin \normts{\z_i^t - \oz^t} 
    + 6L \eta_t^2 \Gamma 
    + \eta_t^2 \frac{\os^2}{n}. \]
\end{lemma}

Lemma \ref{lem:bound_avg_to_optim} bounds the one-step progress of the algorithm \textit{before} a potential server aggregation round.

% :::::::::::::::::::::::::::::::::::::::::::::::::::::::::::::::
\begin{lemma} \label{lem:bound_local_to_avg_z}
    For stepsizes satisfying $\eta_t \leq 2 \eta^{t+H}$, it holds that 
    \[ \Ebb \sumin \norm{\z_i^t - \oz^t}_2^2 \leq \eta_t^2 4 \alpha H n G^2. \]
    with $\alpha = \frac{|\widehat{\lambda}_2|}{1 - |\widehat{\lambda}_2|}$.
\end{lemma}

Lemma \ref{lem:bound_local_to_avg_z} bounds the divergence of the local parameters to their average, which increases through the $H$ iterations in between the server broadcasting rounds.
It is in this process (which involves multiple neighbor averaging steps) where we see the impact of the connectivity of the graph.

% :::::::::::::::::::::::::::::::::::::::::::::::::::::::::::::::
As remarked in Section \ref{sec:model}, $\oz^t = \ox^t \, \forall \, t \notin \Hcal$. 
Otherwise, the equality holds only in expectation (eq. \eqref{eq:exp_oz_is_ox}).
Lemma \ref{lem:difference_averages} bounds the variance of $\oz$ in the latter case.

\begin{lemma} \label{lem:difference_averages}
    For $t \in \Hcal$ and stepsizes satisfying $\eta_t \leq 2 \eta_{t+H}$
    \[ \Ebb \normts{\ox^t - \oz^t} \leq \frac{1}{K} \eta_t^2 4 \alpha H G^2, \]
    with $\alpha$ given in Lemma \ref{lem:bound_local_to_avg_z}.
\end{lemma}

% :::::::::::::::::::::::::::::::::::::::::::::::::::::::::::::::
Lastly, we have the following lemma from \cite{li2019convergence}.
\begin{lemma}\label{lem:technical_lemma_Li}
    Let a sequence $\Delta^t$ satisfy 
    \begin{equation} \label{eq:technical_lemma_li}
        \Delta^{t+1} \leq (1-\mu\eta_t) \Delta^t + \eta_t^2 B 
    \end{equation}
    with $\mu, B > 0$. Then, for a diminishing stepsize $\eta_t = \frac{2}{\mu(\gamma+t)}$ with $\gamma > 0$, it holds that $\Delta_t \leq \frac{v}{\gamma+t}$, where $v = \max\{\frac{4B}{\mu^2},(\gamma+1)\Delta_1\}$.
\end{lemma}
\begin{proof}
    See proof of Theorem 1 in \cite{li2019convergence}.
\end{proof}

This bound establishes the parameter choices that allow the sequence $\Delta^t$ to converge. 
We have now all the tools necessary to prove the main theorem.

% ===============================================================

\begin{proof}[Proof of Theorem \ref{theo:convergence_non-iid_partial_participation}]
    We start by noting that 
    \begin{align*}
        \Ebb \normts{\oz^{t+1} - \zs} &= \Ebb \normts{\oz^{t+1} - \ox^{t+1} + \ox^{t+1} - \zs} \\
        &= \Ebb \normts{\oz^{t+1} - \ox^{t+1}} + \Ebb \normts{\ox^{t+1} - \zs}
        + 2 \Ebb \Abraces{\oz^{t+1} - \ox^{t+1}, \ox^{t+1} - \zs}.
    \end{align*}

    The last term becomes zero when taking expectation, since $\Ebb_{\Scal_t} \oz^t = \ox^t$ (eq. \eqref{eq:exp_oz_is_ox}). 
    The first term is zero when $t+1 \notin \Hcal$, and for all other iterations we can bound it using Lemma \ref{lem:difference_averages} (and the fact that $\eta_{t+1}^2 < \eta_t^2$). 
    We bound the second term using Lemma \ref{lem:bound_avg_to_optim}. We have then
    \[ \Ebb \normts{\oz^{t+1} - \zs} \leq 
        \frac{1}{K} \eta_t^2 4 \alpha H G^2
        + (1-\mu \eta_t) \Ebb \norm{\oz^t - \zs}_2^2
        + \frac{2}{n} \Ebb \sumin \normts{\z_i^t - \oz^t} 
        + 6L \eta_t^2 \Gamma 
        + \eta_t^2 \frac{\os^2}{n}. \]

    Using Lemma \ref{lem:bound_local_to_avg_z} to bound $\Ebb \sumin \normts{\z_i^t - \oz^t}$ we get    
    \[ \Ebb \normts{\oz^{t+1} - \zs} 
        \leq (1-\mu \eta_t) \Ebb \norm{\oz^t - \zs}_2^2
        + \eta_t^2 
        \Sbraces{\paren{\frac{4}{K} + 8} \alpha H G^2 + 6L\Gamma + \frac{\os^2}{n}}. \]

    This has the form of \eqref{eq:technical_lemma_li} with 
    $\Delta^t = \normts{\oz^t - \zs}$ and 
    % $B = \paren{\frac{4}{K} + 8} \alpha H G^2 + 6L\Gamma + \frac{\os^2}{n}$, 
    $B = \paren{\frac{4}{K} + 8} \alpha H G^2 + 6L\Gamma + \os^2/n$, 
    so applying Lemma \ref{lem:technical_lemma_Li}:
    \begin{equation} \label{eq:final_inequality}
        \Ebb \normts{\oz^t - \zs} \leq \frac{v}{\gamma + t} 
        \leq \frac{1}{\gamma + t} \paren{ \frac{4B}{\mu^2} + (\gamma+1)\Delta_1 }.
    \end{equation}
    
    Note that in order to ensure $\eta_t \leq \frac{1}{4L}$ (Lemma \ref{lem:bound_avg_to_optim}) and $\eta_t \leq 2\eta_{t+H}$ (Lemmas \ref{lem:bound_local_to_avg_z} and \ref{lem:difference_averages}) we need to set $\gamma = \max\{8\frac{L}{\mu}-1, H\}$.
    Finally, using $L$-smoothness and $\nabla f(\z^*) = 0$,
    \[ \Ebb[f(\oz^t)] - f(\zs) \leq \frac{L}{2} \Ebb \normts{\oz^t - \zs}. \]
    
    Using \eqref{eq:final_inequality} in the inequality above gives the result.    
\end{proof}
\section{Numerical Results} \label{sec:simulations}

In this section we compare the performance of FedDec with that of FedAvg \cite{mcmahan2017communication} in a problem with partial device participation and heterogeneous data.

We consider the linear regression problem
\[ F_i(\z) = \frac{1}{M} \normts{X_i \, \z - Y_i}, \qquad i \in [n], \]
with $X_i, Y_i \in \Rbb^{M \times d}, M = 10$, and $d=25$.
For generating the regression data we follow a procedure similar to \cite{scaman2017optimal}: we set 
$[X_i]_j \sim \Ncal(0,0.25^2), \, j \in [d]$ and $Y_i = c_i(v + \cos(v))$, 
where $v = X_i \ones$ and $c_i = 2^i, i \in [n]$ is a factor that makes the data at each node significantly different from all others.

For the inter-agent communication, we generate geographic graphs of $n=20$ nodes by taking $n$ points distributed uniformly at random in a $1 \times 1$ square and joining with a link all pairs of points whose Euclidean distance is smaller than a radius $r$. 
We test our algorithms in two graphs with $r=0.35$ and $0.5$, respectively (see Figure \ref{fig:graphs}).
We run $T=5000$ iterations with $K=2$, $m=1$, and $H=10,100$. 

Figure \ref{fig:convergence} shows the convergence of FedDec and FedAvg for each graph in Fig. \ref{fig:graphs} and the two values of $H$.
The stepsize was set to the value indicated in Theorem \ref{theo:convergence_non-iid_partial_participation}. 
The lines shown are the average of ten independent runs of the algorithms on the same problem instance.

\begin{figure}[t]
	\centering
        \begin{subfigure}{0.35\linewidth}
		\centering \includegraphics[width=0.9\linewidth]{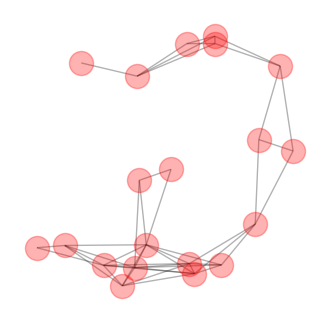}
	\end{subfigure}
  	\begin{subfigure}{0.35\linewidth}
    	\centering \includegraphics[width=0.9\linewidth]{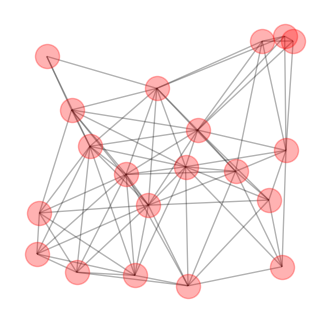}
  	\end{subfigure} 
\caption{Graphs used in the simulations. \textbf{Left:} sparse graph with $r=0.35$. \textbf{Right:} dense graph with $r=0.5$.}
\label{fig:graphs}
\end{figure}

\begin{figure}[t]
	\centering
	\begin{subfigure}{\linewidth}
		\centering \includegraphics[width=0.7\linewidth]{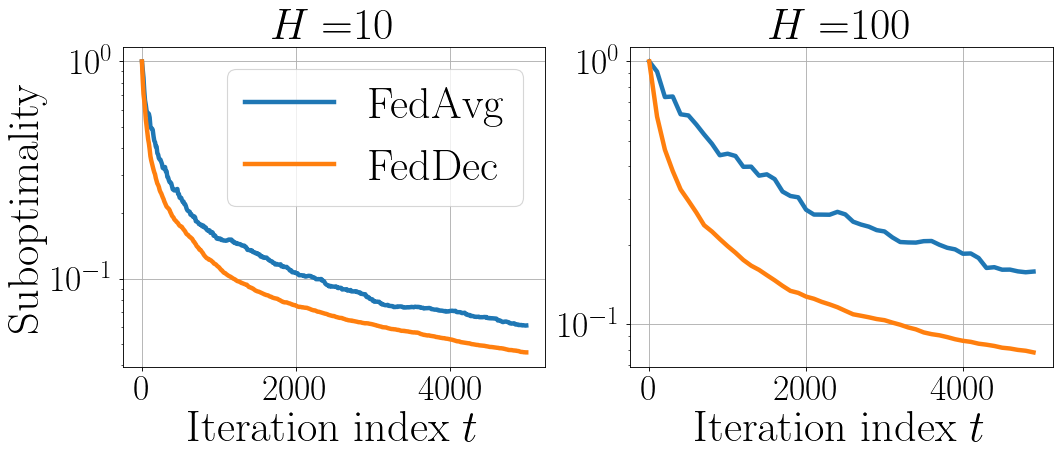}
	\end{subfigure} \\[1ex] 
  	\begin{subfigure}{\linewidth}
    	\centering \includegraphics[width=0.7\linewidth]{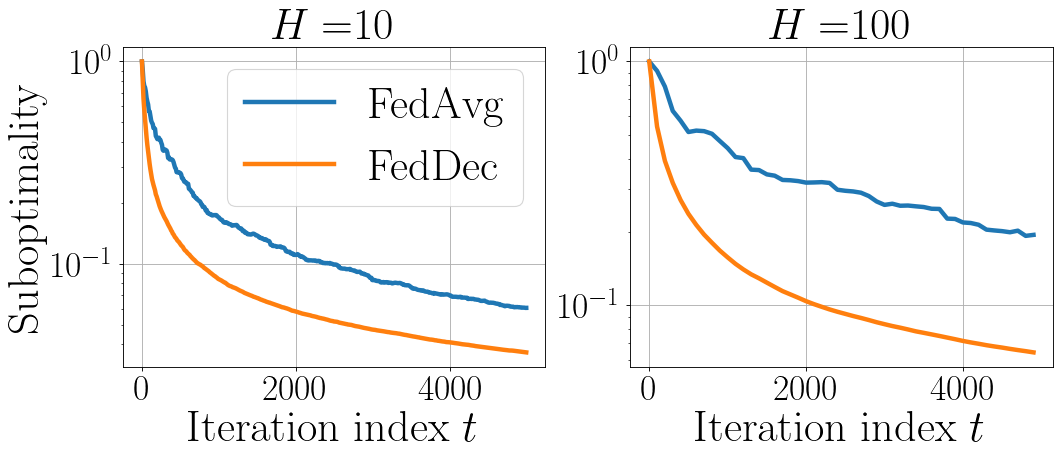}
  	\end{subfigure}
\caption{FedDec versus FedAvg for the sparse graph (top) and the dense graph (bottom). 
When the communication with the server is less frequent (larger $H$), the performance of FedAvg is more degraded than that of FedDec.}
\label{fig:convergence}
\end{figure}

Comparing the  plots in Fig. \ref{fig:convergence} vertically (i.e., comparing the two graphs for the same $H$), 
we confirm that higher connectivity leads to larger gains of FedDec over FedAvg. 
This can be understood intuitively by noticing that a denser graph facilitates a faster spread of information.
Since $|\lambda_2|$ correlates with the graph connectivity (see Table \ref{tab:values_lambda}), in this case it is a good predictor of the convergence speed of FedDec.
However, we note that it has been reported that connectivity, as measured by $|\lambda_2|$, seems to be predictive of the convergence speed of decentralized algorithms (in terms of number of iterations) only when the nodes have sufficiently different data \cite{neglia2020decentralized}, as is the case in our simulations.
When the data is iid among the nodes, the \textit{number of effective neighbors} seems to be a better predictor \cite{vogels2022beyond}.

Comparing the  plots in Fig. \ref{fig:convergence} horizontally (i.e., comparing the $H$ for a given graph), we verify that as $H$ increases the convergence speed of FedAvg decreases more than that of FedDec.
Therefore, FedDec allows for sparser server communication rounds without significantly sacrificing convergence speed, which is in accordance with Theorem \ref{theo:convergence_non-iid_partial_participation}.

One may wonder whether in practice $\alpha$ takes values much smaller than $H$ so that the gains of FedDec can actually be observed.
Assuming a fixed $W$, this question is equivalent to asking what are typical values for $|\lambda_2|^2$ (see Fig. \ref{fig:lambda_dependence}). 
Table \ref{tab:values_lambda} shows this value for many graphs, where $W$ was constructed using the graph's Laplacian \cite{xiao2004fast}.
We computed these values for geographic graphs with different linking radii $r$ and random graphs with different link probabilities $p$, and in both cases, for different number of nodes. 
Geographic graphs are good models for wireless networks \cite{barthelemy2011spatial}, while random graphs have the small-world property of other kinds of networks, such as the Internet \cite{newman2003structure}.
The values of $r$ and $p$ are such that the networks in the same row and column under each graph type have approximately the same number of edges. 
The numbers shown are the average over 10 independent realizations, and the values corresponding to the graphs in Fig. \ref{fig:graphs} are shown in gray. 
We observe that in all cases $|\lambda_2|^2 < 0.9$, which implies $\alpha < 9$. 
Therefore, unless $H$ is particularly small, FedDec is expected to be (potentially, much) faster than FedAvg. 
In particular, for random graphs, which have low network diameter, $|\lambda_2|^2$ decreases more abruptly as connectivity grows.
This further indicates that for well-connected networks, inter-agent communication can make the first term of $B$ in Theorem \ref{theo:convergence_non-iid_partial_participation} become negligible.

\renewcommand{\arraystretch}{1.2}
\begin{table}
\centering
\caption{Values of $|\lambda_2|^2$ for different graphs.}
\label{tab:values_lambda}
\begin{tabular}{|c|c|c|c|}
    \hline
    \multicolumn{4}{|c|}{Geographic graph} \\ \hline
    Connection radius  & $n = 10$   & $n = 20$  & $n = 40$ \\ \hline
    $r = 0.35$     & \textbf{0.78}	& {\cellcolor{gray!25}}\textbf{0.87}      & \textbf{0.83} \\ \hline
    $r = 0.5$      & \textbf{0.7}   & {\cellcolor{gray!25}}\textbf{0.64}     & \textbf{0.56} \\ \hline
    $r = 0.65$	   & \textbf{0.41}	& \textbf{0.33}	 & \textbf{0.34} \\ \hline\hline
    \multicolumn{4}{|c|}{Random graph} \\ \hline
    Link probability & $n = 10$   & $n = 20$  & $n = 40$ \\ \hline
    $p = 0.3$	    & \textbf{0.7}	     & \textbf{0.62}	     & \textbf{0.4} \\ \hline
    $p = 0.5$	    & \textbf{0.42}	     & \textbf{0.29}	     & \textbf{0.17} \\ \hline
    $p = 0.7$	    & \textbf{0.25}	     & \textbf{0.13}	     & \textbf{0.083} \\ \hline
\end{tabular}
\end{table}

\section{Conclusion}

We have presented FedDec, an algorithm that exploits inter-agent communication in FL settings by averaging the agents' parameters with those of their neighbors before each new local SGD update.
We proved that this modification reduces the negative impact of local updates on convergence and that the magnitude of this reduction depends on the spectrum of the graph defining the communication network. 
This further indicates that the effect of local updates and partial device participation can become negligible if the communication network is well-connected. 

This insight suggests that there exists a connectivity threshold where the server does not help convergence anymore. 
Furthermore, we conjecture that for sufficiently dense networks, server communication rounds might even hurt. Future directions include studying this threshold and other trade-offs of peer-to-peer aided FL. 

Overall, exploiting inter-agent communication in FL is a promising way to reduce the frequency of server communication rounds without significantly hurting convergence.

\appendix

\section{Useful Properties}
This section groups a number of facts used in the proofs.

\begin{fact}
    For two vectors $u,v \in \Rbb^d$ and $\delta > 0$ it holds 
    \begin{equation} \label{eq:CS_AM-GM}
        -2\Abraces{\u,\v} \leq \frac{1}{\delta} \normts{\u} + \delta \normts{\v}.
    \end{equation}
    This also holds for matrices $u,v \in \Rbb^{m \times n}$ and the Frobenius norm.
    It can be shown by manipulating the term 
    $\norm{\frac{1}{\delta} \u + \delta \v}_2^2 \geq 0$.
    % using the Cauchy-Shwarz inequality and the inequality of arithmetic and geometric means.
\end{fact}

\begin{fact} 
    For two matrices $A,B \in \Rbb^{m \times n}$ it holds that 
    \begin{equation} \label{eq:bound_frob_A+B_alphas}
        \norm{A+B}_F^2 \leq (1+\alpha^{-1}) \norm{A}_F^2 + (1+\alpha) \norm{B}_F^2.
    \end{equation}
    This can be shown using \eqref{eq:CS_AM-GM}.
\end{fact}

\begin{fact} 
    For a matrix $A \in \Rbb^{m \times n}$ it can be shown that 
    \begin{equation} \label{eq:bound_frob_A-n11}
        \norm{A\paren{I-\noo}}_F^2 \leq \norm{A}_F^2.
    \end{equation}
\end{fact}

\begin{fact} 
    Let $M \in \Rbb^{n \times n}$ be a symmetric matrix satisfying $\ones^T M = \ones$ and having eigenvalues 
    $1 = \lambda_1 > \abs{\lambda_2} \geq \ldots \geq \abs{\lambda_n}$. 
    Then, for a vector $\x \in \Rbb^n$ with the average of its entries denoted $x_{\avg} = \ones^T \x$, it holds
    \begin{equation} \label{eq:bound_lambda_2}
        (\x - x_{\avg} \ones)^T M (\x - x_{\avg} \ones) 
        \leq |\lambda_2| \, \norm{\x - x_{\avg} \ones}_2^2.
    \end{equation}
    This is a consequence of the spectral theorem and the fact that $(\x - x_{\avg} \ones) \perp \text{span}\{\ones\}$.
\end{fact}

\section{Proofs of Lemmas} \label{app:proofs_of_lemmas}
\begin{proof}[Proof of Lemma \ref{lem:bound_avg_to_optim}]

It holds that 

\begin{align*} 
    \Ebb {\normts{\ox^{t+1} - \zs}}
    &= \Ebb {\norm{\oz^t - \frac{\eta_t}{n} \sumin \nabla F_i(\z_i^t,\xi_i^t) - \zs}_2^2} \\
    &= \Ebb \normts{\oz^t - \zs - \frac{\eta_t}{n} \sumin \nabla F_i(\z_i^t)} 
    + \eta_t^2 \; \Ebb \norm{\frac{1}{n} \sumin \nabla F_i(\z_i^t) - \frac{1}{n} \sumin  \nabla F_i(\z_i^t,\xi_i^{j_t})}_2^2 \\
    & \qquad + 2 \; \Ebb \biggl\langle \oz^t - \zs - \frac{\eta_t}{n} \sumin \nabla F_i(\z_i^t), 
    \frac{\eta_t}{n} \sumin \nabla F_i(\z_i^t) - \frac{\eta_t}{n} \sumin  \nabla F_i(\z_i^t,\xi_i^{j_t}) \biggr\rangle
    \numberthis \label{eq:first_bound_avg_to_opt}
\end{align*}
where in the first equality we used that 
\[ \ox^{t+1} = \frac{1}{n} \sumin \x_i^{t+1} 
    = \frac{1}{n} \sumin \sum_{j=1}^n W_{ij} \x_i^{t+\frac{1}{2}} 
    = \frac{1}{n} \sumin \x_i^{t+\frac{1}{2}} \sum_{j=1}^n W_{ij}
    = \ox^{t+\frac{1}{2}} = \oz^t - \frac{\eta_t}{n} \sumin \nabla F_i(\z_i^t,\xi_i^t) \]
and in the second equality we added and subtracted $\frac{\eta_t}{n} \sumin \nabla F_i(\x_i^t)$ inside the norm. 

Note that the last term in \eqref{eq:first_bound_avg_to_opt} is zero, since 
$\Ebb_{\xi_i^t}[\nabla F_i(\z_i^t,\xi_i^t)] = \nabla F_i(\z_i^t)$, 
and the second term is bounded by $\eta_t^2 \frac{\os^2}{n}$ (eq. \eqref{eq:bound_avg_sigma_2}). 
The first term in \eqref{eq:first_bound_avg_to_opt} can be written as (we will apply the expectation directly to the end result)
\begin{equation} \label{eq:term_1}
    \normts{\oz^t - \zs - \frac{\eta_t}{n} \sumin \nabla F_i(\z_i^t)} =
    \normts{\oz^t - \zs}
    + \eta_t^2 \underbrace{ \normts{\frac{1}{n} \sumin \nabla F_i(\z_i^t)}}_{(A)}
    \underbrace{-2 \eta_t \Abraces{\oz^t - \zs, \frac{1}{n} \sumin \nabla F_i(\z_i^t)}}_{(B)}.
\end{equation}

We can bound term $(A)$ with
\begin{align*}
    \normts{\frac{1}{n} \sumin \nabla F_i(\z_i^t)} 
    &\leq \frac{1}{n} \sumin \normts{\nabla F_i(\z_i^t)} \\
    &= \frac{1}{n} \sumin \normts{\nabla F_i(\z_i^t) - \nabla F_i(\z_i^*)} \\
    &\stackrel{\eqref{eq:smooth_df2_leq_ffs}}{\leq} \frac{1}{n} \sumin 2L (F_i(\z_i^t) - F_i(\z_i^*)). \numberthis  \label{eq:bound_norm_grad_Fi}
\end{align*}

On the other hand, term $(B)$ can be bounded as 
\begin{align*}
    -2 \eta_t \Abraces{\oz^t - \zs, \frac{1}{n} \sumin \nabla F_i(\z_i^t)}
    &= \frac{-2 \eta_t}{n} \sumin \Abraces{\oz^t - \z_i^t + \z_i^t - \zs, \nabla F_i(\z_i^t)} \\
    &= \frac{-2 \eta_t}{n} \sumin \Abraces{\oz^t - \z_i^t, \nabla F_i(\z_i^t)}
    - \frac{2 \eta_t}{n} \sumin \Abraces{\z_i^t - \zs, \nabla F_i(\z_i^t)} \\
    &\stackrel{\eqref{eq:CS_AM-GM},\eqref{eq:strong_conv}}{\leq} 
    \frac{\eta_t}{n} \sumin \paren{\frac{1}{\eta_t} \norm{\oz^t - \z_i^t}_2^2 + \eta_t \norm{\nabla F_i(\z_i^t)}_2^2} \\
    &\qquad - \frac{2\eta_t}{n} \sumin \paren{(F_i(\z_i^t)-F_i(\zs)) + \frac{\mu}{2} \norm{\z_i^t - \zs}_2^2}.
\end{align*}
% \endgroup
where we applied \eqref{eq:CS_AM-GM} with the choice $\delta = \eta_t$.

Replacing the bounds on $(A)$ and $(B)$ in \eqref{eq:first_bound_avg_to_opt} gives
\begingroup
\addtolength{\jot}{0.7em}
\begin{align*}
    \norm{\oz^t - \zs - \frac{\eta_t}{n} \sumin \nabla F_i(\z_i^t)}_2^2 
    &\leq \normts{\oz^t - \zs} 
    +  \frac{\eta_t^2}{n} 2L \sumin (F_i(\z_i^t) - F_i(\z_i^*)) + \frac{\eta_t}{n} \sumin \paren{\frac{1}{\eta_t} \norm{\oz^t - \z_i^t}_2^2 + \eta_t \norm{\nabla F_i(\z_i^t)}_2^2} \\
    &\qquad - \frac{2\eta_t}{n} \sumin \paren{(F_i(\z_i^t)-F_i(\zs)) + \frac{\mu}{2} \norm{\z_i^t - \zs}_2^2} \\
    &\stackrel{\eqref{eq:bound_norm_grad_Fi}}{\leq}
    (1-\mu \eta_t) \norm{\oz^t - \zs}_2^2 
    + \frac{1}{n} \sumin \norm{\oz^t - \z_i^t}_2^2 \\
    &\qquad + \underbrace{ \eta_t^2 \frac{4L}{n} \sumin (F_i(\z_i^t) - F_i(\z_i^*)) 
    - \frac{2\eta_t}{n} \sumin (F_i(\z_i^t)-F_i(\zs)) }_{(C)}
\end{align*}
\endgroup

where we used 
\[ -\frac{\eta_t \mu}{n} \sumin \norm{\z_i^t - \zs}_2^2 \leq
- \eta_t \mu \norm{\frac{1}{n} \sumin (\z_i^t - \zs)}_2^2 \\
= - \eta_t \mu \norm{\oz^t - \zs}_2^2. \]

Term $(C)$ was bounded in the proof of Lemma 1 of \cite{li2019convergence}, where for $\eta_t \leq \frac{1}{4L}$ and $\Gamma = \frac{1}{n} \sumin (F_i(\zs) - F_i(\z_i^*))$ they obtained
\[ (C) \leq \frac{1}{n} \sumin \normts{\z_i^t - \oz^t} + 6L \eta_t^2 \Gamma. \]

Replacing these bounds in \eqref{eq:first_bound_avg_to_opt} gives the lemma.
\end{proof}

%%%%%%%%%%%%%%%%%%%%%%%%%%%%%%%%%%%%%%%%%%%%%%%%%%%%%%%%%%%%

\begin{proof}[Proof of Lemma \ref{lem:bound_local_to_avg_z}]

We define 
\[ Z^t = \Sbraces{\z_1^t \cdots \z_n^t} \in \Rbb^{d \times n} \]
\[ \overline{Z}^t = Z^t \noo = \Sbraces{\oz^t \cdots \oz^t } \in \Rbb^{d \times n} \]
\[ \partial F(Z^t,\xi^t) = \Sbraces{\nabla F_1(\z_1^t,\xi_1^t) \cdots \nabla F_n(\z_n^t,\xi_n^t)} \in \Rbb^{d \times n} \]
and $X^t, \overline{X}^t$ analogously. Note that with these definitions, 
$\sumin \norm{\z_i^t - \oz^t}_2^2 = \norm{Z^t - \overline{Z}^t}_F^2$.

We denote $t_b \in \Hcal$ the last time the central server broadcasted the sample average to all nodes so that $\z_i^{t_b} = \z^{t_b} \; \forall i \in [n]$, and define $h := t-t_b \leq (H-1)$.
Note that if $t = t_b$ then $\sumin \norm{\z_i^t - \oz}_2^2 = 0$. Therefore, below we assume $h \geq 1$. 
We have that 
% \begingroup
% \addtolength{\jot}{0.5em}
 \begin{align*}
    \Ebb \norm{Z^t - \overline{Z}^t}_F^2 
    \stackrel{t \notin \Hcal}{=} 
        &\Ebb \norm{X^t - \overline{X}^t}_F^2
    = \Ebb \Big\|
        \underbrace{(X^{t-\frac{1}{2}} - \overline{X}^{t-\frac{1}{2}})}_{Y} 
        W^{t-1} \Big\|_F^2 \\
    &= \Ebb \sum_{i=1}^d \norm{ Y_{[i,:]} W^{t-1} }_2^2 \\
    &= \Ebb \sum_{i=1}^d Y_{[i,:]} \Ebb_W \Sbraces{W W^T} [Y_{[i,:]}]^T \\
    &\stackrel{\eqref{eq:bound_lambda_2}}{\leq} \Ebb \sum_{i=1}^d Y_{[i,:]} 
    \underbrace{ \abs{ \lambda_2 \paren{ \Ebb_W \Sbraces{W W^T} } } }_{|\widehat{\lambda}_2|} \, [Y_{[i,:]}]^T \\
    &= |\widehat{\lambda}_2| \; \Ebb \norm{X^{t-\frac{1}{2}} - \overline{X}^{t-\frac{1}{2}}}_F^2 \\
    &= |\widehat{\lambda}_2| \; \Ebb \Big\| Z^{t-1} - \overline{Z}^{t-1}
     - \eta_{t-1} \partial F(Z^{t-1},\xi^{t-1}) (I - \noo) \Big\|_F^2 \\
    &\stackrel{\eqref{eq:bound_frob_A+B_alphas}}{\leq}
    |\widehat{\lambda}_2| \paren{1+\frac{1}{\alpha}} \Ebb \norm{Z^{t-1} - \overline{Z}^{t-1}}_F^2 
    + |\widehat{\lambda}_2| (1+\alpha) \eta_{t-1}^2 \Ebb \norm{\partial F(Z^{t-1},\xi^{t-1}) (I - \noo) }_F^2 \\
    &\stackrel{\eqref{eq:bound_frob_A-n11}}{\leq}
    |\widehat{\lambda}_2| \paren{1+\frac{1}{\alpha}} \Ebb \norm{Z^{t-1} - \overline{Z}^{t-1}}_F^2
    + |\widehat{\lambda}_2| \ (1+\alpha) \eta_{t-1}^2 \Ebb \norm{\partial F(Z^{t-1},\xi^{t-1})}_F^2 \\
    &=
    |\widehat{\lambda}_2| \paren{1+\frac{1}{\alpha}} \Ebb \norm{Z^{t-1} - \overline{Z}^{t-1}}_F^2 {+} \ |\widehat{\lambda}_2|(1+\alpha) \eta_{t-1}^2 n G^2,
\end{align*}       
% \endgroup
where in the second line we used that $\ones^T W^t = \ones^T$, and in the fourth line, that the matrices $W^t$ are identically distributed independently of the time $t$.
We also used $Y_{[i,:]}$ to indicate the $i$-th row of matrix $Y$.

We can now apply the inequality recursively to get 
\begin{align*}
    \Ebb \norm{Z^t - \overline{Z}^t}_F^2 
    &\leq \paren{1+\frac{1}{\alpha}} |\widehat{\lambda}_2| \; \Ebb \norm{Z^{t-1} - \overline{Z}^{t-1}}_F^2 
    + (1+\alpha) |\widehat{\lambda}_2| \; \eta_{t-1}^2 n G^2 \\ 
    & \leq \paren{1+\frac{1}{\alpha}} |\widehat{\lambda}_2| \; 
        \Bigg[ \paren{1+\frac{1}{\alpha}} |\widehat{\lambda}_2| \; 
        \Ebb \norm{Z^{t-2} - \overline{Z}^{t-2}}_F^2
    + (1+\alpha) |\widehat{\lambda}_2| \; \eta_{t-2}^2 n G^2 \Bigg] 
        + (1+\alpha) |\widehat{\lambda}_2| \; \eta_{t-1}^2 n G^2  \\
    &= \Sbraces{\paren{1+\frac{1}{\alpha}} |\widehat{\lambda}_2| \;}^2 
        \Ebb \norm{Z^{t-2} - \overline{Z}^{t-2}}_F^2 
    + (1+\alpha) |\widehat{\lambda}_2| \; n G^2 \sum_{i=1}^2 \Sbraces{\paren{1+\frac{1}{\alpha}} |\widehat{\lambda}_2| \;}^{i-1} \eta_{t-i}^2 \\
    &\leq \Sbraces{\paren{1+\frac{1}{\alpha}} |\widehat{\lambda}_2| \;}^h \Ebb \norm{Z^{t_b} - \overline{Z}^{t_b}}_F^2 
     + (1+\alpha) |\widehat{\lambda}_2| \; n G^2 \sum_{i=1}^h \Sbraces{\paren{1+\frac{1}{\alpha}} |\widehat{\lambda}_2| \;}^{i-1} \eta_{t-i}^2.
\end{align*}       

We note that the first term is zero, since at broadcasting time $Z^{t_b} = \overline{Z}^{t_b}$.
We now set $\alpha = \frac{|\widehat{\lambda}_2|}{1 - |\widehat{\lambda}_2|}$ so that the expression between square brackets takes value 1. Therefore, 
\[ \Ebb \norm{Z^t - \overline{Z}^t}_F^2 
\leq \alpha n G^2 H \eta_{t_b}^2 
\leq \alpha n G^2 H 4 \eta_t^2 \]
where we have used that for the choice of $\alpha$ given above it holds $(1+\alpha) |\widehat{\lambda}_2| \; = \alpha$, and that the stepsizes $\eta_t$ are monotonically decreasing and satisfy $\eta_t \leq 2 \eta_{t+H}$. 
\end{proof}

Note that in the result above we could have used $(H-1)$ instead of $H$. However, since  this bound is used again in the proof of Lemma \ref{lem:difference_averages} for $h=H$, we loosen it slightly here to be able to apply it directly in the next proof.

%%%%%%%%%%%%%%%%%%%%%%%%%%%%%%%%%%%%%%%%%%%%%%%%%%%%%%%%%%%%

\begin{proof}[Proof of Lemma \ref{lem:difference_averages}]

For $t \in \Hcal$ we have that 
\[ \Ebb \normts{\oz^t - \ox^t} = \Ebb \normts{\frac{1}{K} \sum_{\ell=1}^K \x_{j_\ell}^t - \ox^t}  \\
    = \frac{1}{K^2} \sum_{\ell=1}^K \Ebb \Sbraces{ \EbbSt \normts{  \x_{j_\ell}^t - \ox^t} }
    = \frac{1}{Kn} \sumin \Ebb \normts{  \x_i^t - \ox^t} \]

where we used that for independent $X_i$, $\text{Var}(\sum_i X_i) = \sum_i \text{Var}(X_i)$, 
and that for a random variable $X \in \Omega$ with a discrete probability density function $X \sim f_X$ and a function $g: \Omega \ram \Rbb$ it holds $\Ebb[g(X)] = \sum_x g(x) f_{X}(x)$.

Since $\sumin \normts{ \x_i^t - \ox^t} = \big\|X^t - \overline{X}^t \big\|_F^2$ we can repeat the procedure done in Lemma \ref{lem:bound_local_to_avg_z} to bound this term.
\end{proof}

\bibliographystyle{ieeetr}
\bibliography{bibliography}

\end{document}